\documentclass[conference]{IEEEtran}
\IEEEoverridecommandlockouts

\usepackage{cite}
\usepackage{amsmath,amssymb,amsfonts}
\usepackage{graphicx}
\usepackage{booktabs}
\usepackage{tabularx}
\usepackage{url}
\usepackage{xcolor}
\definecolor{projectlinkblue}{HTML}{1D63AD}
\usepackage{soul}
\usepackage{stfloats}
\usepackage[hidelinks]{hyperref}
\hypersetup{
    pdftitle={Composing Learned Robot Behaviors with Temporal Logic at Runtime},
    pdfauthor={Moritz Zoellner, Anastasios Manganaris, Ahmed H. Qureshi, Rohan Paleja},
    pdfsubject={Preprint}
}

\begin{document}

\title{Composing Learned Robot Behaviors with Temporal Logic at Runtime}
\author{\IEEEauthorblockN{Moritz Zoellner,
Anastasios Manganaris,
Ahmed H. Qureshi, and
Rohan Paleja}
\thanks{\raggedright Authors are with the Department of Computer Science,
Purdue University.\protect\newline
\texttt{\{zoellner,amangana,ahqureshi,rpaleja\}@purdue.edu}.}}
\maketitle

\begin{abstract}
Executing Linear Temporal Logic (LTL) instructions with learned robot policies faces two practical challenges: demonstrations may cover individual behaviors without containing the temporal compositions requested at deployment, and semantic success predicates often provide little useful motor guidance through quantitative robustness.
We address these challenges by separating behavior learning from temporal composition.
From the same offline demonstrations, we train an unconditioned multimodal diffusion policy and a semantic predictor that estimates which semantic outcomes are likely to follow a proposed action sequence.
These predictions provide a learned guidance signal without requiring hand-designed robustness measures for semantic task objectives.
At deployment, an LTL automaton tracks instruction progress and scores the policy's proposed actions according to their predicted semantic outcomes.
Neither learned component receives the specification during training, allowing demonstrated behaviors to be reused in longer, previously unseen temporal compositions without retraining.
For additional runtime safety constraints with informative continuous margins, an optional robustness-based controller locally refines the selected actions.
Experiments in a navigation environment, CALVIN manipulation, and on a real robot demonstrate reliable execution of complex temporal instructions and additional safety constraints, with substantial improvements over temporal-logic baselines.
For supplementary material, visit \href{https://moritz-zoellner.com/learned-ltl}{\textcolor{projectlinkblue}{our project page}}.
\end{abstract}

\section{Introduction}

\suppressfloats[t]
\begin{figure}[t]
    \centering
    \setlength{\abovecaptionskip}{2pt}
    \includegraphics[width=\columnwidth]{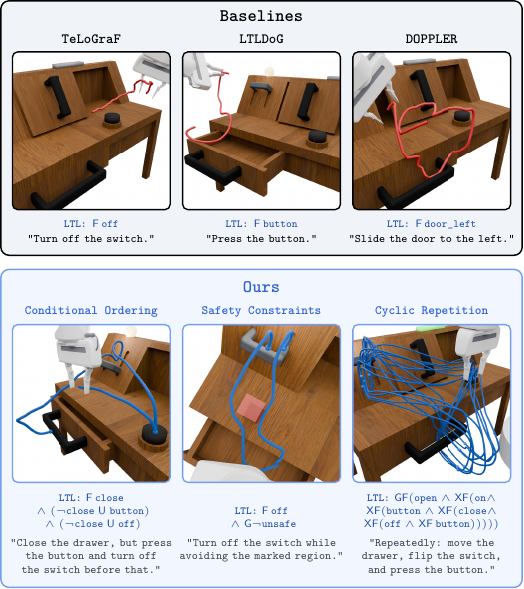}
    \caption{\textbf{From demonstrations to runtime temporal composition.}
    Our method learns motor behaviors and their semantic outcomes from offline
    demonstrations, then uses an LTL automaton at deployment to compose them into
    new instructions. \textbf{Top:} representative failures of baselines when
    executing semantic manipulation objectives. \textbf{Bottom:} representative
    executions of our method for LTL specifications that were not present in the data.}
    \label{fig:results_highlights}
    \vspace{-16pt}
\end{figure}

Linear Temporal Logic (LTL) provides a structured interface for specifying robot behavior beyond conventional task conditioning~\cite{kress2009temporal,baier2008principles}.
It can express non-Markovian behavior, constraints introduced at inference time, including safety requirements, and explicit temporal structure over desired outcomes.
A desirable capability is thus to train a behavior policy from data and subsequently execute an LTL specification $\phi$ supplied only at runtime.
However, learning policies that can execute such runtime-specified temporal-logic instructions remains challenging in complex robot-manipulation environments~\cite{feng2024ltldog,pmlr-v267-meng25b,feng2025doppler}. We attribute this difficulty to two fundamental bottlenecks: the combinatorial scaling of temporal compositions across tasks and the absence of continuous progress signals for abstract semantic goals. In this work, we address these bottlenecks and enable zero-shot execution of complex instructions, as shown in  Fig.~\ref{fig:results_highlights}.

Importantly, temporal instructions can recombine behaviors that are individually well represented in the training data, while the complete deployment specification may never appear during training.
For example, a robot may have demonstrations for pressing a button, opening a drawer, and flipping a switch, while deployment requests a new ordering or combination of these same behaviors.
Existing learned LTL methods such as TeLoGraF~\cite{pmlr-v267-meng25b} and DOPPLER~\cite{feng2025doppler} condition action generation or value estimation on the specification itself, which introduces a coverage problem as the number of possible temporal compositions grows; the number of required demonstrations has been observed to scale combinatorially with task horizon~\cite{mandlekar2020learning}.
Instead of attempting to scale training data to cover this combinatorial explosion, we decouple the behavior generation from the logical specification. We learn from demonstrations of primitive behaviors and provide a runtime mechanism that composes these seen primitives into previously unseen temporal instructions at runtime.

A second challenge is finding a suitable connection between the LTL specification and the robot actions that will satisfy it.
Existing temporal-logic methods often rely on quantitative robustness~\cite{donze2010robust} as an optimization signal for satisfying a specification~\cite{feng2024ltldog}.
This works well when the predicate admits a meaningful continuous notion of progress, such as distance to a target or margin from an unsafe region.
Many manipulation propositions are different: they describe a semantic outcome rather than how the robot should move to achieve it.
For example, if \texttt{light\_on} is represented only by whether the light is on or off, its robustness remains unchanged across all unsatisfied states; barely missing the switch will return the same signal as moving in the opposite direction.
Such a predicate therefore provides no informative gradient back to the motor behavior until additional task-specific shaping or knowledge is introduced.
Thus, even when a proposition is easy to define and evaluate, it may provide little useful signal for learning or selecting the behavior that will satisfy it.

In this paper, we address these challenges to make LTL guidance from offline data feasible in complex domains like robot manipulation (Fig.~\ref{fig:results_highlights}).
We train two complementary components from the same demonstration data: a multimodal diffusion policy~\cite{chi2025diffusion} that learns the motor behaviors themselves, and a semantic prediction model that learns which outcomes those behaviors are likely to produce.
Rather than relying on a hand-designed robustness signal for semantic task progress, this learned prediction provides an informative score even for abstract propositions without a natural continuous notion of closeness.
At deployment, the associated LTL automaton tracks progress through the specification $\phi$ and evaluates the predicted semantic outcomes to determine which proposed behavior advances the current instruction. 
The temporal composition is introduced only at runtime, allowing isolated behaviors to be executed within complex, unseen specifications.
Finally, for additional runtime safety constraints that do admit meaningful continuous margins, we retain robustness as an optional local refinement mechanism.
We evaluate our method in increasingly challenging settings previously unsolvable with existing LTL policy learning methods.
Figure~\ref{fig:results_highlights} summarizes the resulting capabilities in CALVIN~\cite{mees2022calvin}, where our approach is able to reliably execute complex LTL instructions that were never present as complete tasks in the training data, while baseline methods struggle to achieve even simple instructions.

\section{Related Work}

\textbf{Temporal-logic planning and control.}
Temporal logic has long been used to specify temporally structured robot tasks~\cite{kress2009temporal}. Among classical approaches to temporal-logic based control, we are particularly influenced by automaton-based controllers that track task progress and support receding-horizon execution~\cite{ding2014ltlrecedinghorizoncontrol}. For some temporal logics, quantitative robustness provides a continuous measure of specification satisfaction~\cite{donze2010robust}.
We build on both automata and robustness using learned semantic predictions to evaluate proposed behaviors, rather than planning in discrete transition models as done in \cite{ding2014ltlrecedinghorizoncontrol}.

\textbf{Learning controllers and reusable skills.}
Reinforcement-learning approaches train controllers using rewards derived from a supplied LTL specification~\cite{hasanbeig2018logically}.
Temporal Logic Imitation instead composes reusable mode-based motion policies according to LTL plans~\cite{pmlr-v205-wang23a}.
Related neuro-symbolic approaches learn symbolic effects and policies from demonstrations for subsequent planning~\cite{pmlr-v205-silver23a}.
Our method predicts the semantic outcomes of action sequences sampled from one multimodal behavior policy, rather than selecting from an explicitly constructed skill or operator library.

\textbf{Generalizable LTL execution from offline data.}
LTLDoG guides generated trajectories using quantitative formula-satisfaction scores, evaluated directly or approximated by a learned formula-conditioned evaluator~\cite{feng2024ltldog}.
TeLoGraF conditions trajectory generation on an encoded Signal Temporal Logic (STL) formula and learns from specification--demonstration pairs~\cite{pmlr-v267-meng25b}.
DOPPLER generates short segments and composes them through receding-horizon execution, using a formula-conditioned critic trained by temporal-difference learning~\cite{feng2025doppler}.
In contrast, our method learns next-semantic-event predictions from labeled demonstrations without formula supervision.

\section{Problem Setting and Preliminaries}
\label{sec:preliminaries}

We are given a dataset of state-action demonstrations
$\mathcal D=\{\tau_i\}_{i=1}^N$, with
$\tau_i=(s_0^i,a_0^i,s_1^i,\ldots,a_{T_i-1}^i,s_{T_i}^i)$,
together with a fixed vocabulary of atomic propositions $\mathrm{AP}$.
A known labeling function $L:S\rightarrow 2^{\mathrm{AP}}$ maps each
state to the propositions that hold in that state, and we denote the
resulting semantic label by $\ell_t=L(s_t)$.
In practice, $l_t$ is a binary vector of indicators for task-relevant propositions, such as $\texttt{light\_on}$, that are readily available from simulators and semantic perception systems in the real world.
Each demonstration can therefore be viewed as an action sequence
together with its observed semantic evolution,
$
\tau_i=
\big((s_0^i,\ell_0^i),a_0^i,\ldots,
a_{T_i-1}^i,(s_{T_i}^i,\ell_{T_i}^i)\big).
$

The demonstrations show how the robot realizes individual semantic
behaviors represented in $\mathrm{AP}$, but need not contain the longer
compositional behaviors requested at deployment.

Robustness-based temporal-logic guidance methods such as
LTLDoG~\cite{feng2024ltldog} additionally associate each proposition
$p\in\mathrm{AP}$ with a quantitative predicate
$r_p:S\rightarrow\mathbb R$, where $r_p(s)>0$ indicates satisfaction.
This value can score and, when differentiable, guide behavior toward
satisfying a proposition. Such guidance is informative only when
$r_p$ exposes meaningful progress before the proposition becomes true.
For example,
\[
\begin{gathered}
\underbrace{r_{\mathrm{reach}}(s)=\epsilon-\lVert x_{\mathrm{robot}}-x_{\mathrm{goal}}\rVert}_{\text{progress is observable}} \\
\text{vs.} \\
\underbrace{r_{\mathrm{light}}(s)=2*\,\mathbf 1[\texttt{light\_on}\in L(s)]-1}_{\text{success only}}.
\end{gathered}
\]
For the latter,
$
\rho(\mathbf F\,\texttt{light\_on},\tau)=\max_t r_{\mathrm{light}}(s_t)=-1
$
for every trajectory on which the light remains off. Consequently,
two unfinished trajectories $\tau^a$ and $\tau^b$ receive the same
robustness,
$
\rho(\mathbf Fp,\tau^a)=\rho(\mathbf Fp,\tau^b)=-1,
$
even if $\tau^a$ moves toward the switch while $\tau^b$ moves away.
Thus, a semantic proposition can be easy to evaluate while providing
no quantitative information about which motor behavior is progressing
toward satisfying it. Because the gradient is zero almost everywhere, constructing a useful optimization signal requires extensive task-specific reward shaping. Instead, we seek to learn this guidance directly
from labeled demonstrations.

\begin{figure*}[b]
    \centering
    \vspace{-18pt}
    \setlength{\abovecaptionskip}{2pt}
    \includegraphics[width=\textwidth]{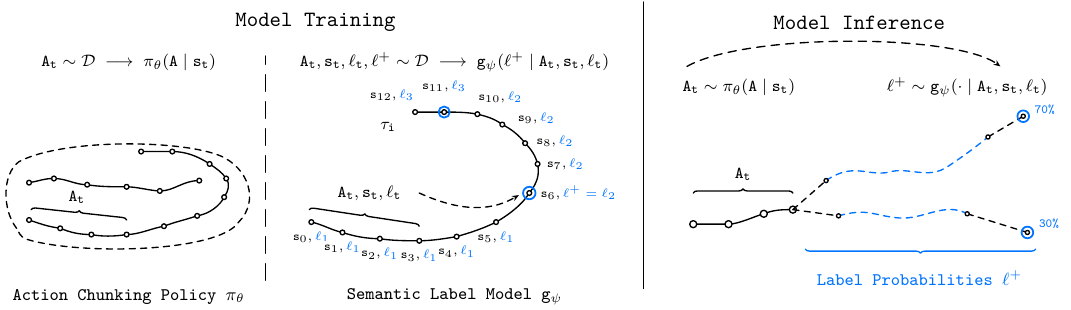}
    \vspace{-18pt}
    \caption{\textbf{Learning semantic outcome predictions from demonstrations.}
    \textbf{Left:} an unconditioned multimodal action-chunking policy learns
    the motor behaviors present in the demonstration data.
    \textbf{Middle:} the same demonstrations supervise the semantic predictor
    by pairing each executed action chunk with the next distinct semantic label
    reached later in the trajectory.
    \textbf{Right:} at inference time, the predictor maps a proposed action chunk
    to a distribution over possible future semantic outcomes, allowing the
    controller to estimate what that behavior is likely to accomplish even
    before the corresponding proposition becomes true.
    Curves and probabilities are illustrative.}
    \label{fig:hint2_semantic_learning}
\end{figure*}

At deployment, the controller receives an LTL specification $\phi$ over the fixed proposition vocabulary $\mathrm{AP}$ and must generate an execution whose induced semantic label sequence satisfies it.
We use standard LTL syntax over propositions $p\in\mathrm{AP}$, Boolean operators $\neg,\wedge,\vee$, and temporal operators $\mathrm{X}$, $\mathrm{F}$, $\mathrm{G}$, and $\mathrm{U}$. We assume this specification can be translated into a deterministic B\"uchi automaton $\mathcal A_\phi=(Q,\Sigma,\delta,q_0,F),$ defined by a set of automaton states $Q$, a set of semantic labels $\Sigma=2^{\mathrm{AP}}$, a transition function $\delta:Q\times\Sigma\to Q$, an initial state $q_0 \in Q$, and a set of accepting states $F\subseteq Q$.
Using this automaton, a semantic label sequence $\ell_0, \ell_1, \ldots $ observed from the controller's behavior induces a run of automaton states $q_0, q_1, \ldots$, with each state transition obeying $q_{t+1} = \delta(q_t,\ell_t)$. For the specification $\phi$ to be satisfied, this sequence of automaton states must intersect infinitely often with the set of accepting states $F$~\cite{baier2008principles}. The state $q_t$ can be understood as recording all task progress made up until time $t$, and the automaton structure exposes the next labels to achieve at each state so further progress is made toward accepting states.

In satisfying $\phi$, many learned temporal-logic methods condition actions or value estimation directly on the specification itself.
To illustrate the resulting coverage problem, consider only ordered instructions containing $K$ semantic events,
\[
\begin{aligned}
\Phi_K^{\mathrm{seq}} = \big\{&\mathbf F\!\left(p_1\land\mathbf F(p_2\land\cdots\mathbf F p_K)\right) \\
                          &{}: p_i\in\mathrm{AP}\big\}.
\end{aligned}
\]
Letting $M=|\mathrm{AP}|$, even this restricted family contains
$
|\Phi_K^{\mathrm{seq}}|=M^K
$
possible instructions when propositions may repeat.
A specification-conditioned method learns from data of the form
$\mathcal D_{\mathrm{spec}}=\{(\tau_i,\phi_i)\}_{i=1}^N$.
Obtaining direct training coverage of every instruction in
$\Phi_K^{\mathrm{seq}}$ would therefore require at least $M^K$
distinct training specifications.
Our key insight is that $\phi$ need not be part of the training supervision.
We instead learn state-action priors and their semantic consequences strictly from labeled demonstrations. The temporal specification is introduced only at deployment, where it aids the composition of pre-learned behaviors on-the-fly.

\section{Method}
\label{sec:method}

We start from a pretrained multimodal diffusion policy $\pi_\theta(A\mid o_t)$ that, from the current observation, can propose a diverse set of short-horizon behaviors (i.e., action chunks) represented in the demonstration data.
While this policy captures the individual behaviors available to the robot, it lacks the temporal reasoning required to satisfy a long-horizon specification $\phi$.
We therefore separate the problem into two questions: what semantic outcome is each proposed behavior likely to produce, and which of those outcomes advances the current task? We learn a semantic outcome model to answer the first, and use the LTL automaton to answer the second by selecting action chunks that progress toward the accepting state of the LTL automaton.

\textbf{Learning semantic outcome predictions.}
Rather than evaluating task progress at every physical time step, we train our semantic outcome model to predict the next distinct semantic label to occur in the demonstration (Fig. ~\ref{fig:hint2_semantic_learning}). From the demonstration dataset, we sample 
\[
\begin{gathered}
(s_t,\ell_t,A_t,\ell_t^+)\sim\mathcal D,\\
\ell_t^+=\ell_{\tau^+},\quad
\tau^+=\min\{\tau>t:\ell_\tau\neq\ell_t\}.
\end{gathered}
\]
Here, $A_t=(a_t,\ldots,a_{t+h-1})$ is the executed $h$-action prefix, and $\ell_t^+$ denotes the next distinct semantic label occurring later in the same demonstration.
We train a semantic outcome model
$\hat{\ell}_t^+=g_\psi(s_t,\ell_t,A_t)\in[0,1]^{|\mathrm{AP}|}$
using binary cross-entropy objective, displayed in Eq. \ref{eq:bin}.
\begin{equation}
\mathcal L_{\mathrm{sem}}(\psi)
=-\sum_{p\in\mathrm{AP}}\bigl[
\ell_p^+\log\hat{\ell}_p^+
+(1-\ell_p^+)\log(1-\hat{\ell}_p^+)
\bigr]    
\label{eq:bin}
\end{equation}

Each semantic label $\ell$ is a binary vector over the atomic propositions in $\mathrm{AP}$.
Accordingly, each entry of the prediction $\hat{\ell}_t^+$ estimates whether one particular atomic proposition will hold in the next distinct semantic label.
Crucially, this target is not restricted to the action prefix $A_t$ itself: the next label change may occur later in the demonstration, beyond the chunk used as input.
The model therefore learns to associate a short movement with the semantic outcome that eventually follows it in the demonstrated behavior.
Through this objective, $g_\psi$ provides a learned, quantitative measure of how different candidate behaviors progress toward abstract semantic outcomes, bypassing the need for hand-crafted robustness functions.

Figure~\ref{fig:hint2_semantic_learning} exemplifies this process.
Both the action-chunking policy and semantic outcome model are trained from the same demonstration data $\mathcal D$: the policy proposes an action chunk $A_t$, while $g_\psi$ learns which propositions are likely to hold in the next distinct semantic label reached later in the demonstration following that chunk.
At inference time, the multimodal policy can propose multiple action chunks $A_1,\ldots,A_K$ from the same observation, corresponding to different possible behaviors.
Applying the semantic outcome model to each chunk gives a corresponding predicted semantic-outcome vector $\hat{\ell}_1^+,\ldots,\hat{\ell}_K^+$.
We can therefore distinguish which semantic outcomes the different policy behaviors are heading toward.
What remains is to determine which of these predicted outcomes is useful for satisfying the current specification $\phi$.

\begin{figure*}[b]
    \centering
    \vspace{-14pt}
    \setlength{\abovecaptionskip}{2pt}
    \includegraphics[width=\textwidth]{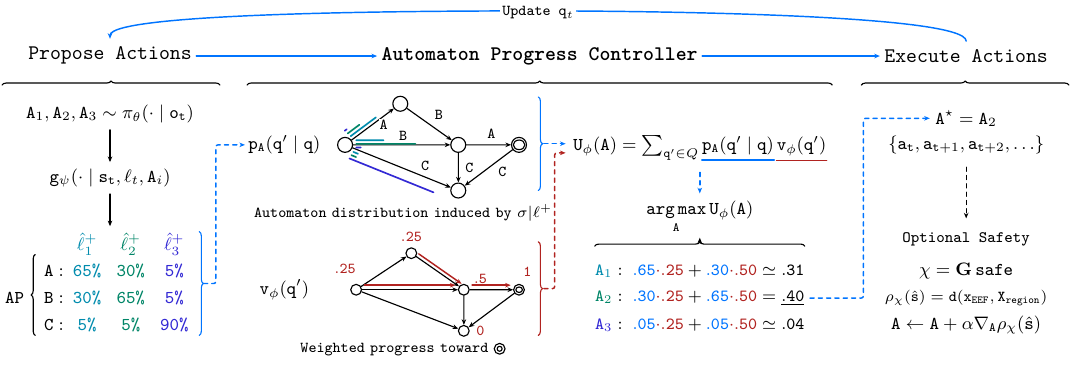}
    \vspace{-24pt}
    \caption{\textbf{Full overview of the runtime temporal-composition controller.}
    The figure walks through one action-selection step for an example with atomic
    propositions $\mathrm{AP}=\{A,B,C\}$, $K=3$ candidate action chunks, and an LTL
    automaton whose accepting state is indicated by a double circle.
    \textbf{Left:} the behavior policy proposes three candidate action chunks,
    and $g_\psi$ maps each candidate to a semantic prediction $\hat{\ell}_k^+$,
    whose entries give the predicted probabilities that the individual atomic
    propositions hold in the next distinct semantic label.
    From each $\hat{\ell}_k^+$, we construct a factorized distribution over complete
    Boolean labels $\sigma\in\Sigma=2^{\mathrm{AP}}$.
    \textbf{Center:} each possible label $\sigma$ induces a transition from the
    current automaton state, producing a distribution over next automaton states
    that is scored by progress toward acceptance using $U_\phi$.
    In the illustrated example, the second candidate receives the highest score.
    \textbf{Right:} the controller selects this candidate for execution and can
    optionally refine it using robustness for an additional runtime safety constraint.
    After the selected chunk is executed, the observed semantic label updates the
    automaton state $q_t$, the policy proposes new candidates, and the process
    repeats in closed loop.}
    \label{fig:hint2_method}
\end{figure*}

\textbf{Semantic predictions for automaton-guided control.}
To determine the next semantic outcome to pursue, we translate the deployment specification $\phi$ into a deterministic B\"uchi automaton as defined in Section~\ref{sec:preliminaries}.
Classical LTL control methods have used progress through the specification automaton as a control objective or constraint, typically assuming an explicit model of the system transitions~\cite{ding2014ltlrecedinghorizoncontrol}.
We adopt the same idea of continuous automaton progress, but replace these known system transitions with our learned semantic outcome predictions over the policy's proposed behaviors.
To guide the robot, we compute for each automaton state a progress potential
$v_\phi(q)=\gamma^{d_\phi(q)}$, where $d_\phi(q)$ is the shortest distance from $q$ to an accepting state, and $\gamma \in (0,1)$ is a discount factor.
Thus, states closer to satisfying the specification receive larger potential, accepting states receive $v_\phi(q)=1$; rejecting states from which acceptance is unreachable receive $v_\phi(q)=0$.

For each candidate action chunk $A$, $g_\psi$ produces a prediction $\hat{\ell}^+\in[0,1]^{|\mathrm{AP}|}$.
We denote by $\hat{\ell}_p^+$ the entry corresponding to atomic proposition $p$, i.e., the predicted probability of $p$ holding in the next distinct semantic label.
The automaton, however, transitions on complete Boolean labels $\sigma\in\Sigma=2^{\mathrm{AP}}$.
Therefore, from $\hat{\ell}^+$, we construct a factorized distribution $P_\psi(\sigma\mid s_t,\ell_t,A)$: for a label $\sigma$, we multiply $\hat{\ell}_p^+$ for propositions present in $\sigma$ and $1-\hat{\ell}_p^+$ for those absent. Given the current automaton state $q_t$ and the probability for every possible label $\sigma$ conditioned on the current state $s_t$, current label $\ell_t$, and action chunk $A$, we can compute a conditional distribution over next automaton states,
\[
P(q'\mid  q_t, s_t, \ell_t, A)
=\sum_{\sigma\in\Sigma}P_\psi(\sigma\mid s_t,\ell_t,A)\,
\mathbf 1\!\left[\delta(q_t,\sigma)=q'\right].
\]
We use this distribution to score each sampled action chunk according to the potential of the automaton states it may reach,
\[
U_\phi(A)=\sum_{q'\in Q}P(q'\mid q_t, s_t, \ell_t, A)\,v_\phi(q').
\]
Finally, we execute the highest-scoring candidate,
$
A^\star=\arg\max_{A\in\{A_1,\ldots,A_K\}}U_\phi(A).
$
Intuitively, we select the action chunk whose likely semantic outcomes make the greatest immediate progress toward an accepting state of the automaton.

\textbf{Optional Inference-time Safety Refinement.}
Some constraints are only specified at inference time and therefore cannot generally be encoded into the fixed training-time proposition vocabulary.
For example, the particular region a robot should avoid may only become known when the instruction is given.
We denote such an additional runtime constraint by $\chi$.
These constraints typically restrict how a behavior should be executed (e.g., $\mathrm{G}\neg\texttt{unsafe}$) rather than which semantic outcome should be reached. Because automaton progress alone cannot provide useful local guidance for these geometric constraints, we rely on quantitative robustness as a natural continuous signal~\cite{zoellner2026temporal}.
When an inference-time constraint $\chi$ is present, we refine the selected chunk $A^\star$ using a learned local dynamics model $\hat s_{k+1}=f_\eta(\hat s_k,a_k)$ and gradient ascent on the robustness of the predicted rollout:
\[
A^\star\leftarrow A^\star+\alpha\nabla_A\rho_\chi(\hat s_{0:H}).
\]
This allows constraints introduced only at deployment to locally modify the selected behavior without requiring them to be learned as semantic propositions during training.

\textbf{Control System Overview.}
Figure~\ref{fig:hint2_method} summarizes the resulting closed-loop controller and provides an example of the flow of information between its components.
At inference time, the policy proposes multiple action chunks, the semantic model predicts their future proposition changes, and $\phi$ is compiled into an automaton.
The runtime controller maintains the current automaton state $q_t$ and uses our automaton-potential score $U_\phi$ to rank the proposed chunks by their predicted progress from that state.
After selection and optional refinement, the controller executes the selected chunk and observes the resulting semantic label to update $q_t$. This process repeats in closed loop, with recurring instructions continuing beyond individual visits to accepting states.

\section{Experiments}
\label{sec:experiments}

\begin{figure}[b]
    \centering
    \vspace{-18pt}
    \setlength{\abovecaptionskip}{2pt}
    \includegraphics[width=0.98\columnwidth]{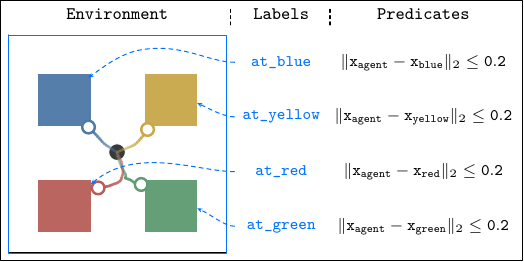}
    \caption{\textbf{2D navigation environment and semantic vocabulary.}
    The point agent is trained from single-goal demonstrations to reach four
    colored regions. Each region induces one atomic proposition through the
    distance-based predicates shown on the right; the same predicates are used
    by all methods.}
    \label{fig:toy_squares_environment}
\end{figure}

\begin{figure}[t]
    \centering
    \setlength{\abovecaptionskip}{4pt}
    \begin{minipage}[t]{\columnwidth}
        \vspace{0pt}
        \includegraphics[width=\linewidth]{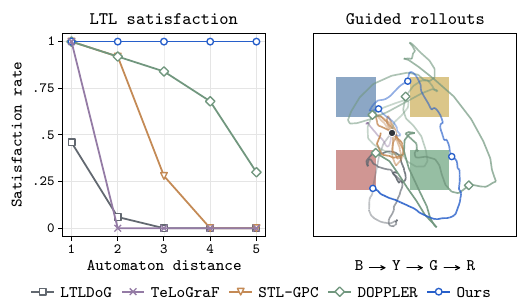}
    \end{minipage}\par\vspace{3pt}
    \begin{minipage}[t]{\columnwidth}
        \vspace{0pt}
        \centering
        \footnotesize
        {\fontencoding{T1}\fontfamily{lmtt}\fontsize{8.5}{10}\selectfont
        Broader temporal instructions}\par\vspace{2pt}
        \setlength{\tabcolsep}{3pt}
        \renewcommand{\arraystretch}{1.12}
        \begin{tabularx}{\linewidth}{@{}c*{4}{>{\centering\arraybackslash}X}@{}}
            \toprule
            \textbf{Spec.} & \textbf{LTLDoG} & \textbf{STL-GPC} & \textbf{DOPPLER} & \textbf{Ours} \\
            \midrule
            $\phi_1$ & 0\% & 0\% & 18\% & \textbf{100\%} \\
            $\phi_2$ & 0\% & 0\% & 12\% & \textbf{100\%} \\
            $\phi_3$ & 0\% & 0\% & 60\% & \textbf{100\%} \\
            \bottomrule
        \end{tabularx}
    \end{minipage}
    \caption{\textbf{Composition from primitive demonstrations.}
    \textbf{Top:} executed LTL satisfaction for increasingly long prefixes of
    the sequence blue $\rightarrow$ yellow $\rightarrow$ green $\rightarrow$
    red $\rightarrow$ blue, together with representative four-goal rollouts.
    \textbf{Bottom:} success on three structurally different specifications
    using the same trained models:\\
    $\phi_1=(\neg g\,\mathbf U\,r)\wedge(\neg g\,\mathbf U\,b)
    \wedge\mathbf F\bigl(g\wedge\mathbf F(y\wedge\mathbf F(r\wedge
    \mathbf F(b\wedge\mathbf F(y\wedge\mathbf F g))))\bigr)$,\\
    $\phi_2=\mathsf S(r,y,g,b,r,y,g,b)\vee\mathsf S(g,b,y,r,g,b,y,r)$,\\
    $\phi_3=\mathbf G\mathbf F r\wedge\mathbf G\mathbf F b
    \wedge\mathbf G\mathbf F g\wedge\mathbf G\mathbf F y$.\\
    Here, $r,b,g,y$ denote contact with the corresponding colored regions, and
    $\mathsf S(p_1,\ldots,p_m)=\mathbf F(p_1\wedge\mathbf X\mathbf F(p_2
    \wedge\cdots\wedge\mathbf X\mathbf F p_m))$ denotes strict ordered visitation.
    Recurrence is evaluated over two completed rounds.}
    \label{fig:results_toy_squares}
    \vspace{-15pt}
\end{figure}

We evaluate our method in three settings with specific intentions.
First, a controlled \textbf{2D navigation environment} which isolates whether individually demonstrated behaviors can be composed into longer and structurally varied temporal instructions.
Second, the \textbf{CALVIN} domain which evaluates guidance toward semantic manipulation outcomes and their composition into complex LTL instructions.
Finally, \textbf{real-robot experiments} which evaluate physical execution of learned behaviors under runtime specifications.
We additionally assess the optional safety controller through an ablation with and without robustness refinement. We compare against four baselines spanning the primary state-of-the-art approaches in learned temporal-logic execution:
\begin{itemize}[\setlength{\IEEElabelindent}{0pt}\setlength{\labelsep}{0.45em}]
    \item \textbf{LTLDoG}~\cite{feng2024ltldog} generates state-action trajectories and guides them using quantitative LTL satisfaction.
    \item \textbf{TeLoGraF}~\cite{pmlr-v267-meng25b} learns specification-conditioned trajectory generation from paired demonstrations and temporal-logic specifications.
    \item \textbf{DOPPLER}~\cite{feng2025doppler} composes short behavior segments using a specification-conditioned value function trained through temporal-difference learning.
    \item \textbf{STL-GPC}, a custom adaptation of GPC~\cite{qi2026gpc}, provides a matched candidate-ranking baseline: it samples behaviors from the same policy, predicts their state trajectories, and ranks them using temporal-logic robustness rather than our learned semantic outcomes.
\end{itemize}

\vspace{-3pt}
\subsection{Composition from primitive demonstrations}
\label{sec:primitive-composition}

Figure~\ref{fig:toy_squares_environment} gives an overview of the 2D navigation environment together with the four atomic propositions and distance-based predicates used by all methods.
We collect 10,000 scripted demonstrations in which the point agent moves from randomized initial states to one of the four target regions, such that the dataset covers each primitive reaching behavior but no temporal compositions of them.
Our method and STL-GPC use the same unconditioned diffusion policy adapted from robomimic~\cite{mandlekar2021robomimic}; at inference time, we sample $K=32$ candidate 8-action chunks.
The semantic predictor is a lightweight 3,404-parameter MLP with separate encoders for state, current label, and action chunk, trained with the BCE objective from Sec.~\ref{sec:method}.
We translate each LTL specification \(\phi\) into its automaton using Spot \cite{duret2022spot}.
LTLDoG, TeLoGraF, DOPPLER, and STL-GPC are adapted to the same environment and demonstration data, with robustness-based methods using the same quantitative predicates shown in Fig.~\ref{fig:toy_squares_environment}.
We evaluate each supported method over 50 executed rollouts per specification.
This setting isolates whether behaviors demonstrated only individually can be composed into temporal instructions that never appear as complete trajectories in the training data.

The results show that our method maintains 100\% satisfaction as the requested instruction grows from one to five goals, despite training only on single-goal demonstrations (Fig. \ref{fig:results_toy_squares}, top).
LTLDoG and TeLoGraF are largely limited to the shortest instructions and rapidly lose success once the requested execution extends beyond the single-goal trajectories seen during training.
STL-GPC performs better in this environment because the distance-based predicates in Fig.~\ref{fig:toy_squares_environment} provide an informative robustness signal that directly reflects progress toward each region.
DOPPLER is more robust to composition: its formula-conditioned critic bootstraps value across successive short segments, allowing it to combine behaviors across demonstrations rather than requiring each full trajectory to appear in the data.
Its performance nevertheless decreases as the automaton distance grows.
In contrast, our controller applies the same local decision rule at every automaton state: once a semantic goal is reached and the automaton advances, the problem resets to a short-horizon search over candidate behaviors. Consequently, longer instructions require more automaton transitions, but they do not demand a longer learned prediction horizon or a larger model capacity.

\begingroup
\renewcommand{\bottomfraction}{0.55}
\begingroup
\makeatletter\def\ftype@figure{4}\makeatother
\suppressfloats[b]
\begin{figure}[b]
    \centering
    \vspace{-18pt}
    \setlength{\abovecaptionskip}{2pt}
    \includegraphics[width=\columnwidth]{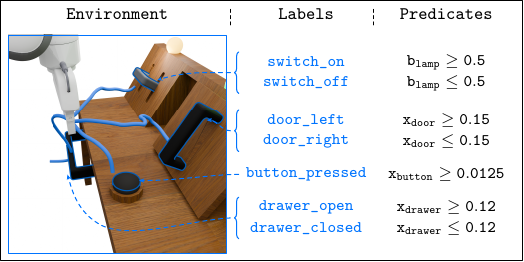}
    \caption{\textbf{CALVIN manipulation environment and semantic vocabulary.}
    The robot interacts with a switch, button, drawer, and sliding door.
    The right panel lists the atomic propositions and simulator-state predicates
    used to label these behaviors throughout the experiments.}
    \label{fig:calvin_environment}
\end{figure}
\endgroup

\begin{figure*}[b]
    \centering
    \vspace{-10pt}
    \setlength{\abovecaptionskip}{2pt}
    \includegraphics[width=\textwidth]{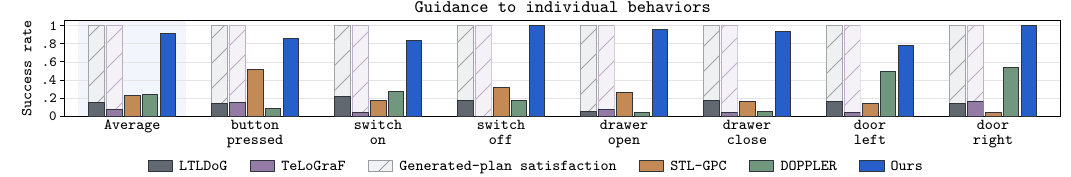}
    \par\vspace{4pt}
    \begingroup
    \footnotesize
    \setlength{\tabcolsep}{3pt}
    \setlength{\jot}{0pt}
    \setlength{\aboverulesep}{1.5pt}
    \setlength{\belowrulesep}{2.5pt}
    \renewcommand{\arraystretch}{1.05}
    \renewcommand{\tabularxcolumn}[1]{m{#1}}
    \begin{tabularx}{\textwidth}{@{}c>{\fontsize{6}{7}\selectfont}l
        *{4}{>{\centering\arraybackslash}X}@{}}
        \toprule
        & {\footnotesize\textbf{LTL specification}}
        & {\scriptsize\textbf{LTLDoG}} & {\scriptsize\textbf{STL-GPC}}
        & {\scriptsize\textbf{DOPPLER}} & {\scriptsize\textbf{Ours}} \\
        \midrule
        \textbf{selection} &
        $\mathbf{F}(\texttt{door\_left} \vee \texttt{switch\_off}
            \vee \texttt{button\_pressed})$
        & 38 & 72 & 50 & \textbf{94} \\
        \midrule
        \textbf{unordered} &
        $\mathbf{F}\,\texttt{door\_left}
            \wedge \mathbf{F}\,\texttt{switch\_off}
            \wedge \mathbf{F}\,\texttt{button\_pressed}$
        & 2 & 0 & 6 & \textbf{84} \\
        \midrule
        \textbf{conditional} &
        $\begin{aligned}
            &\mathbf{F}\,\texttt{drawer\_close}
                \wedge (\neg\texttt{drawer\_close}\ \mathbf{U}\ \texttt{button\_pressed}) \\
            &\quad \wedge (\neg\texttt{drawer\_close}\ \mathbf{U}\ \texttt{switch\_off})
        \end{aligned}$
        & 0 & 0 & 0 & \textbf{100} \\
        \midrule
        \textbf{chained} &
        $\mathbf{F}(\texttt{button\_pressed}
            \wedge \mathbf{X}\mathbf{F}(\texttt{door\_right}
            \wedge \mathbf{X}\mathbf{F}(\texttt{switch\_off}
            \wedge \mathbf{X}\mathbf{F}\,\texttt{drawer\_close})))$
        & 0 & 2 & 0 & \textbf{100} \\
        \midrule
        \textbf{branched} &
        $\begin{aligned}
            &\mathbf{F}(\texttt{button\_pressed}
                \wedge \mathbf{X}\mathbf{F}(\texttt{drawer\_close}
                \wedge \mathbf{X}\mathbf{F}\,\texttt{switch\_off})) \\
            &\quad \vee \mathbf{F}(\texttt{switch\_off}
                \wedge \mathbf{X}\mathbf{F}(\texttt{drawer\_close}
                \wedge \mathbf{X}\mathbf{F}\,\texttt{button\_pressed}))
        \end{aligned}$
        & 0 & 0 & 0 & \textbf{100} \\
        \midrule
        \textbf{cyclic} &
        $\mathbf{G}\mathbf{F}(\texttt{drawer\_open}
            \wedge \mathbf{X}\mathbf{F}(\texttt{switch\_on}
            \wedge \mathbf{X}\mathbf{F}(\texttt{drawer\_close}
            \wedge \mathbf{X}\mathbf{F}\,\texttt{switch\_off})))$
        & 0 & 0 & 0 & \textbf{88} \\
        \midrule
        \multicolumn{2}{@{}l}{\textbf{Mean: composite instructions}}
        & \textbf{6.67} & \textbf{12.33} & \textbf{9.33} & \textbf{94.33} \\
        \midrule
        \textbf{region} &
        $\mathbf{F}\,\texttt{switch\_off}
            \wedge \mathbf{G}\neg\texttt{unsafe\_region}$
        & 16 & 2 & N/A & \textbf{98}\,(42) \\
        \midrule
        \textbf{angle} &
        $\mathbf{F}\,\texttt{door\_right}
            \wedge \mathbf{G}\,\texttt{tcp\_tilt\_20deg}$
        & 0 & 0 & N/A & \textbf{82}\,(0) \\
        \midrule
        \textbf{gripper} &
        $\mathbf{F}\,\texttt{drawer\_close}
            \wedge \mathbf{G}\,\texttt{gripper\_open}$
        & 0 & 0 & N/A & \textbf{98}\,(0) \\
        \midrule
        \multicolumn{2}{@{}l}{\textbf{Mean: runtime constraints}}
        & \textbf{5.33} & \textbf{0.67} & N/A & \textbf{92.67}\,(14) \\
        \bottomrule
    \end{tabularx}
    \par\endgroup
    \caption{\textbf{Semantic guidance and temporal composition in CALVIN.}
    \textbf{Top:} success on seven individual semantic goals.
    Hatched background bars show the success rate of generated state trajectories
    for LTLDoG and TeLoGraF.
    \textbf{Bottom:} full LTL satisfaction (\%) on six composite instructions and
    three runtime safety constraints.
    Parentheses report results without safety refinement; N/A denotes configurations not evaluated.
    The same trained behavior policy and semantic predictor are reused across
    all specifications; safety rows additionally activate the optional robustness
    refinement. Each tested method--task pair uses 50 executed rollouts.}
    \label{fig:results_calvin}
    \vspace{-10pt}
\end{figure*}

\begin{figure}[t]
    \centering
    \setlength{\abovecaptionskip}{2pt}
    \includegraphics[width=\columnwidth]{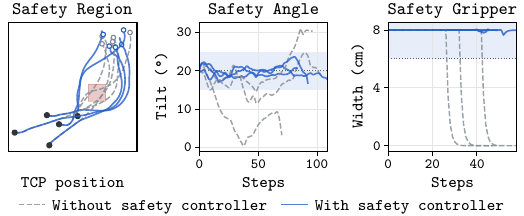}
    \caption{\textbf{Effect of optional safety refinement.}
    Example executions with and without refinement for region, orientation,
    and gripper constraints. Joint task-and-constraint success increases from
    42\%, 0\%, and 0\% without refinement to 98\%, 82\%, and 98\%, respectively.}
    \label{fig:calvin_safety_ablation}
    \vspace{-18pt}
\end{figure}

This behavior also extends beyond simple ordered liveness chains.
Figure~\ref{fig:results_toy_squares} (bottom) evaluates prerequisite ordering with continuation ($\phi_1$), long branching ($\phi_2$), and recurrence ($\phi_3$) using the same trained policy and semantic predictor.
Our method achieves 100\% satisfaction on all three specifications without retraining, including the cyclic task.
These results show that our method can robustly use primitive behavior data to support structurally different temporal instructions introduced at deployment.

\subsection{Semantic guidance and composition in CALVIN}
\label{sec:calvin-semantic-composition}

The navigation setting provides a geometric measure of progress toward each goal. We next test whether guidance remains effective when the goals instead describe semantic outcomes of object interactions.
Figure~\ref{fig:calvin_environment} shows the CALVIN manipulation environment and the semantic vocabulary used throughout our evaluation.
We train an unconditioned diffusion policy on six hours of human play from the CALVIN-D dataset~\cite{mees2022calvin}.
The policy observes proprioception together with wrist and third-person RGB images and predicts action trajectories with horizon 16, of which 8 actions are executed before replanning.
Our semantic predictor instead operates on low-dimensional simulator state, together with the current semantic label and executed action chunk, and predicts the next-label distribution using the BCE objective from Sec.~\ref{sec:method}.
Our method reuses the same trained behavior policy and semantic predictor for all tested specifications; only the deployment formula $\phi$ changes.
The baseline methods are adapted to the same environment and dataset.
\par\endgroup

We first evaluate the seven semantic behaviors as single-eventuality specifications $\mathbf Fp$. This isolates whether each method can select motor behavior that reaches one desired semantic outcome, which is a prerequisite for executing more complex specifications.
As Fig.~\ref{fig:calvin_environment} shows, these propositions are defined primarily from simulator object state rather than from the robot configuration itself.
This exposes the limitation of trajectory guidance based on such semantic predicates.
For LTLDoG and TeLoGraF, the generated state trajectories satisfy the requested proposition in 100\% of trials, while their actual execution success remains low.
In other words, the planner can generate a state sequence in which the light turns off or the drawer closes, but the accompanying motor actions do not reliably cause that state change when executed.
The robustness objective rewards the desired object state, but provides no information about how the robot must move to physically produce it.
Our semantic predictor instead learns this connection directly from demonstrations, associating proposed action chunks with the semantic outcomes they tend to produce.
As Fig.~\ref{fig:results_calvin} (top) shows, this substantially improves execution success across the individual semantic goals.

\suppressfloats[b]
\begin{figure}[b]
    \centering
    \vspace{-18pt}
    \setlength{\abovecaptionskip}{2pt}
    \includegraphics[width=\columnwidth]{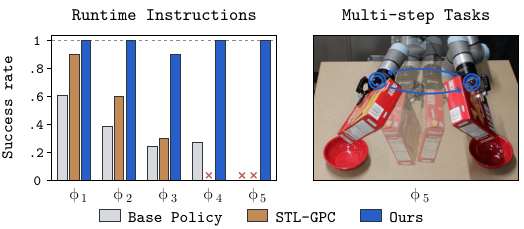}
    \caption{\textbf{Physical execution of runtime temporal instructions.}
    For STL-GPC, the robustness predicates are
    $r_{\texttt{box\_grabbed}}=g$,
    $r_{\texttt{pour\_left}}=-\theta-90^\circ$, and
    $r_{\texttt{pour\_right}}=\theta-90^\circ$,
    where $\theta$ is the wrist angle and $g=1$ for a closed gripper, $-1$ otherwise.
    Success rates for five runtime instructions:\\
    {\fontsize{6}{7}\selectfont
    $\phi_1=\mathbf F(\texttt{box\_grabbed}\wedge\mathbf F\,\texttt{pour\_left})$;\\[-1.5pt]
    $\phi_2=\mathbf F(\texttt{box\_grabbed}\wedge\mathbf F\,\texttt{pour\_right})$;\\[-1.5pt]
    $\phi_3=\phi_1\wedge\mathbf G\neg\texttt{unsafe}$;\\[-1.5pt]
    $\phi_4=\phi_1\wedge(\texttt{upright}\,\mathbf U\,\texttt{near\_bowl})$;\\[-1.5pt]
    $\phi_5=\mathbf F\bigl(\texttt{box\_grabbed}\wedge
    \mathbf G\mathbf F(\texttt{pour\_left}\wedge\mathbf F\,\texttt{pour\_right})\bigr)$.}\\
    Each condition in $\phi_1$--$\phi_4$ uses $N=20$ real-robot trials;
    our method achieves 100\% success on $\phi_5$ over $N=5$ trials
    (evaluated for three full cycles).}
    \label{fig:real_world_results}
\end{figure}

Having established composition from primitive demonstrations in navigation and guidance toward individual semantic outcomes in CALVIN, we now combine both capabilities in six composite manipulation instructions.
As shown in Fig.~\ref{fig:results_calvin} (bottom), our method achieves high full-instruction satisfaction across selection, unordered, conditional, chained, branched, and cyclic specifications, with a mean success rate of 94.3\%.
This success follows directly from our decoupled architecture: once the semantic predictor reliably identifies behaviors leading toward individual propositions, the automaton allows for sequencing those behaviors according to the current instruction state.
The baselines perform substantially worse on the multi-behavior specifications; since they already struggle to achieve individual behaviors, performance largely collapses once satisfying $\phi$ requires several semantic behaviors in sequence or under temporal constraints.
Only with our approach can the same learned behavioral knowledge be reused across substantially different temporal instructions at deployment.

Next, we evaluate the optional safety controller on three runtime constraints: avoiding a specified region while turning off the switch, maintaining a target end-effector angle while moving the door, and keeping the gripper open while closing the drawer.
These constraints are not part of the semantic label vocabulary used to train the task controller, but unlike the abstract manipulation outcomes considered above, they admit meaningful continuous robustness measures.
Figure~\ref{fig:calvin_safety_ablation} illustrates how refinement modifies the selected behavior to respect the additional constraint.
Across the three tasks, mean joint satisfaction increases from 14\% without refinement to 92.7\% with refinement.
This illustrates the complementary role of robustness in our framework: it is useful for local constraints with informative geometric margins, even when those constraints are introduced only at deployment.

\subsection{Physical execution}
\label{sec:physical-execution}

Finally, we evaluate whether our method extends beyond simulation in a real-world UR5e manipulation setting.
We collect 130 demonstrations in which the robot picks up a box of Cheez-Its, pours into the left bowl, the right bowl, or both bowls, and then places the box back down.
The visuomotor diffusion policy operates directly on camera observations, while FoundationPose~\cite{wen2024foundationpose} tracks the position and orientation of the Cheez-It box; we use this estimated state to compute the semantic labels during both training and deployment.
From the same demonstrations, we train the semantic predictor and local dynamics model.
We evaluate five LTL specifications combining the three demonstrated behaviors, including single-goal guidance ($\phi_1$, $\phi_2$), runtime safety ($\phi_3$, $\phi_4$), and a challenging cyclic repetition task ($\phi_5$).

The real-world results in Fig.~\ref{fig:real_world_results} show that our method successfully executes behavior-selection, runtime-safety, and cyclic instructions on the physical robot.
STL-GPC performs substantially worse, showing that robustness is a weaker guidance signal for selecting the appropriate semantic behavior in this setting.
The cyclic task especially highlights the mechanism of our approach: although the demonstrations contain only the constituent pouring behaviors, the controller composes them at runtime to complete this long-horizon sequence, which would be costly to demonstrate.
This demonstrates the value of separating behavior learning from temporal composition: in real-world settings, where collecting demonstrations is expensive, our method allows the robot to learn the behaviors present in the data and recombine them given substantially richer LTL instructions at deployment.

\section{Discussion}

In this paper, we identify two practical obstacles to learning LTL execution from offline robot data: limited coverage of complete temporal specifications and the lack of informative robustness signals for abstract semantic propositions.
We address both by learning robot behaviors and their semantic outcomes from demonstrations, then composing these behaviors through an LTL automaton only at runtime.
This substantially extends LTL execution to more complex manipulation settings and to real-world robot behavior.

The current controller, however, is deliberately greedy: it predicts only the next semantic outcome and selects actions according to immediate automaton progress.
We therefore do not guarantee satisfaction of every supported specification.
A natural extension is to predict the next $N>1$ semantic transitions rather than only one.
Increasing $N$ provides more lookahead into future obligations, but also requires the demonstration data to contain longer behavioral compositions.
This opens a useful trade-off between highly reusable local predictions at $N=1$ and richer but more data-dependent temporal prediction for larger $N$.

Furthermore, $g_\psi$ predicts marginal probabilities for individual propositions, while the controller constructs the joint label distribution by factorizing these marginals.
This does not capture correlations between proposition values and can become inaccurate when automaton transitions depend on particular combinations of propositions.
Learning structured joint semantic predictions is a natural extension.

Finally, our current method assumes a fixed vocabulary of atomic propositions with a known labeling function during training.
This excludes behaviors whose semantics are not already grounded, such as open-ended instructions involving previously unspecified objects or relations.
Future work could learn this grounding directly from demonstration data, e.g., through visual observations and language annotations, providing a path toward LTL execution beyond a manually defined proposition set and privileged low-dimensional state, and ultimately toward a more end-to-end vision-based system.

\bibliographystyle{IEEEtran}
\bibliography{references}

\end{document}